\documentclass[graybox]{svmult}

\usepackage{amsfonts}       %
\usepackage{amssymb}        %
\usepackage{amsmath}        %
\usepackage{url}            %
\usepackage{color}          %
\usepackage{subfigure}      %
\usepackage[linesnumbered,vlined,ruled]{algorithm2e}    %
\usepackage{pgf}            %
\usepackage{tikz}           %
\usetikzlibrary{arrows,automata}    %

\usepackage{mathptmx}       %
\usepackage{helvet}         %
\usepackage{courier}        %
\usepackage{type1cm}        %
\usepackage{makeidx}         %
\usepackage{graphicx}        %
\usepackage{multicol}        %
\usepackage[bottom]{footmisc}%

\renewcommand{\qedsymbol}{\hfill$\blacksquare$}

\makeatletter
\newcommand*{\defeq}{\mathrel{\rlap{%
                     \raisebox{0.3ex}{$\m@th\cdot$}}%
                     \raisebox{-0.3ex}{$\m@th\cdot$}}%
                     =}
\makeatother

\newtheorem{Procedure}[theorem]{Procedure}
\newcommand{\newexample}[2]{\hskip -\marginparsep {\bf Example #1.} {\it #2}}
\newcommand{\examplerevisited}[2]{\hskip -\marginparsep {\bf Example #1 Revisited.} {\it #2}}
\newcommand{\myremark}[3]{\hskip -\marginparsep {\bf Remark #1 ({\bf #2}).} {\it #3}}

\newcommand{\RMB}{\mathrm{B}}

\newcommand{\RMR}{\mathrm{R}}
\newcommand{\RMT}{\mathrm{T}}

\newcommand{\CF}{\mathcal{F}}

\newcommand{\CI}{\mathcal{I}}
\newcommand{\CL}{\mathcal{L}}

\newcommand{\CQ}{\mathcal{Q}}
\newcommand{\CR}{\mathcal{R}}
\newcommand{\CS}{\mathcal{S}}
\newcommand{\CT}{\mathcal{T}}

\newcommand{\BFB}{\mathbf{B}}

\newcommand{\BFR}{\mathbf{R}}

\newcommand{\BFT}{\mathbf{T}}

\newcommand{\BFW}{\mathbf{W}}

\graphicspath{{fig/}}

\newcommand{\BBN}{\mathbb{N}}

\newcommand{\BBT}{\mathbb{T}}

\newcommand{\Next}{\mathbf{X}}
\newcommand{\Always}{\mathbf{G}}
\newcommand{\Event}{\mathbf{F}}
\newcommand{\Until}{\mathcal{U}}
\newcommand{\Implies}{\Rightarrow}

\newcommand{\notltl}{\neg}
\newcommand{\andltl}{\wedge}
\newcommand{\orltl}{\vee}

\newcommand{\prop}{\mathtt{p}}
\newcommand{\opt}{\pi}
\newcommand{\optrun}{\textsc{Optimal-Run}\ }

\newcommand{\buchi}{B\"uchi\ }
\newcommand{\ie}{{\it i.e.},~}

\makeindex             %

\begin{document}

\title*{Optimal Multi-Robot Path Planning with LTL Constraints: Guaranteeing Correctness Through Synchronization}
\titlerunning{Optimal Multi-Robot Path Planning with LTL Constraints and Synchronization}
\author{Alphan Ulusoy, Stephen L. Smith, and Calin Belta}
\institute{Alphan Ulusoy, Calin Belta \at Boston University, Boston, MA, USA \email{alphan@bu.edu, cbelta@bu.edu}
\and Stephen L. Smith \at University of Waterloo, Waterloo, ON, Canada \email{stephen.smith@uwaterloo.ca}}
\maketitle

\abstract{
In this paper, we consider the automated planning of optimal paths for a robotic team satisfying a high level mission specification.
Each robot in the team is modeled as a weighted transition system where the weights have associated deviation values that capture the non-determinism in the traveling times of the robot during its deployment.
The mission is given as a Linear Temporal Logic (LTL) formula over a set of propositions satisfied at the regions of the environment.
Additionally, we have an optimizing proposition capturing some particular task that must be repeatedly completed by the team.
The goal is to minimize the maximum time between successive satisfying instances of the optimizing proposition while guaranteeing that the mission is satisfied even under non-deterministic traveling times.
Our method relies on the communication capabilities of the robots to guarantee correctness and maintain performance during deployment.
After computing a set of optimal satisfying paths for the members of the team, we also compute a set of synchronization sequences for each robot to ensure that the LTL formula is never violated during deployment.
We implement and experimentally evaluate our method considering a persistent monitoring task in a road network environment.
}

\section{Introduction}
\label{sec:intro}
Temporal logics~\cite{EAE:1990}, such as Linear Temporal Logic (LTL) and Computation Tree Logic (CTL), are extensions of propositional logic that can capture temporal relations.
Even though temporal logics have been used in model checking of finite systems~\cite{CB-JPK:2008} for quite some time, they have gained popularity as a means for specifying complex mission requirements in path planning and control synthesis problems only recently~\cite{PT-GJP:2006, MK-CB:2008, TW-UT-RMM:2010}.
Existing work on path planning and control synthesis concentrates on LTL specifications for finite state systems, which may be abstractions of their infinite counterparts~\cite{PT-GJP:2006,JT-BY-CB-IC-JB:2010}.
Particularly, given the system model and the mission specification expressed in some temporal logic, satisfying paths and corresponding control strategies can be computed automatically through a search of the state space for deterministic~\cite{MK-CB:2010}, non-deterministic~\cite{WT:2002,JT-BY-CB-IC-JB:2010,MK-CB:2008,HKG-GF-GFP:2007} and probabilistic systems~\cite{AB-LDA:1995,MK-GN-DP:2002,XCD-SLS-CB-DR:2011}.

However, more often than not, there are multiple paths that can satisfy a given mission specification.
In that case, one generally wants to be able to pick the path that is superior to others with respect to some metric, such as safety, speed, cost, etc.
In our previous work, we focused on mission specifications given in LTL along with a particular cost function, and proposed an automated method for finding optimal robot paths that satisfy the mission and minimize the cost function for a single robot~\cite{SLS-JT-CB-DR:2011}. 
Next, we extended this approach to multi-robot teams by utilizing an abstraction based on timed automata~\cite{AU-SLS-XCD-CB-DR:2011.b}.
Then, we proposed a robust method that could accomodate uncertainties in the traveling times of robots with limited communication capabilities~\cite{AU-SLS-XCD-CB:2011}.

Extending the optimal path planning problem from a single robot to multiple robots is not trivial, as the joint asynchronous motion of all members of the team must be captured in a finite model.
In~\cite{MK-CB:2010}, the authors propose a method for decentralized motion of multiple robots subject to LTL specifications.
Their method, however, results in sub-optimal performance as it requires the robots to travel synchronously, blocking the execution of the mission before each transition until all robots are synchronized.
The vehicle routing problem (VRP)~\cite{PT-DV:2001} and its extensions to more general classes of temporal constraints~\cite{SK-EF:2008.b,SK-EF:2008} also deal with finding optimal satisfying paths for a given specification.
In~\cite{SK-EF:2008}, the authors consider optimal vehicle routing with metric temporal logic specifications by converting the problem to a mixed integer linear program (MILP).
However, their method does not apply to the missions where robots must repeatedly complete some task, as it does not allow for specifications of the form ``always eventually''.
Furthermore, none of these methods are robust to timing errors that can occur during deployment, as they rely on the robots' ability to follow generated trajectories exactly for satisfaction of the mission specification.

In~\cite{AU-SLS-XCD-CB-DR:2011.b}, we proposed a method that uses timed automata to capture the joint asynchronous motion of the members of the robotic team in the environment.
After providing a bisimulation~\cite{RM:1989} of an infinite-dimensional timed automaton to a finite dimensional transition system, we applied our results from~\cite{SLS-JT-CB-DR:2011} to compute an optimal satisfying run.
However, multi-robot paths found using this method are implementable only if the traveling times of the robots during deployment exactly match the traveling times used for planning.
Otherwise, the order of events may switch resulting in the violation of the mission specification during deployment.
In~\cite{AU-SLS-XCD-CB:2011}, we addressed this issue for robots operating under communication constraints that limit their communication capabilities to a subset of regions.
We showed that a trace-closed mission specification will never be violated due to uncertainties in the speeds of the robots.
Then, we proposed a synchronization protocol to maintain and characterize the field performance of the robotic team.

The methods given in ~\cite{AU-SLS-XCD-CB-DR:2011.b} and~\cite{AU-SLS-XCD-CB:2011} are actually two extremes:
In~\cite{AU-SLS-XCD-CB-DR:2011.b}, the robots can follow the generated trajectories exactly and do not communicate at all, while in~\cite{AU-SLS-XCD-CB:2011} the robots' traveling times during deployment deviate from those used in planning, and they cannot communicate freely.
In this paper, we address the middle between these two extremes: the robots cannot follow the generated trajectories exactly, but they can communicate regardless of their positions in the environment.
Thus, after obtaining an optimal satisfying run of the team, we compute synchronization sequences that leverage the communication capabilities of the robots to robustify the planned trajectory against deviations in traveling times.

The main contribution of this paper is threefold.
First, we provide an algorithm to capture the joint asynchronous behavior of a team of robots modeled as transition systems in a single transition system.
This team transition system is provably more compact than the approach based on timed automata that we previously proposed in \cite{AU-SLS-XCD-CB-DR:2011.b}.
Second, for a satisfying run made up of a finite length prefix and an infinite length cyclic suffix, we propose a synchronization protocol and an algorithm to compute synchronization sequences that guarantee correctness under non-deterministic traveling times that may be observed during deployment.
Finally, we provide an automated framework that leverages these two methods along with the \optrun algorithm previously proposed in \cite{SLS-JT-CB-DR:2011} to solve the multi-robot optimal path planning problem with robustness guarantees.
Our experiments show that the computed runs and synchronization sequences indeed provide robustness to uncertainties in traveling times that may occur during the deployment of the team.

The rest of the paper is organized as follows:
In Sec.~\ref{sec:prelim}, we provide some definitions and preliminaries in formal methods.
In Sec.~\ref{sec:prob}, we formulate the optimal and robust multi-robot path planning problem and give an outline of our approach.
We provide a complete solution to this problem in Sec.~\ref{sec:sol}.
In Sec.~\ref{sec:exp}, we present experiments involving a team of robots performing a persistent surveillance mission in a road network environment.
Finally, in Sec.~\ref{sec:conc}, we conclude with final remarks.

\section{Preliminaries}
\label{sec:prelim}
In this section, we introduce the notations that we use in the rest of the paper and briefly review some concepts related to automata theory, LTL, and formal verification.
For a more rigorous treatment of these topics, we refer the interested reader to~\cite{EMC-DP-OG:1999,JEH-RM-JDU:2007,CB-JPK:2008} and references therein.

For a set $\Sigma$, we use $|\Sigma|$, $2^\Sigma$, $\Sigma^*$, and $\Sigma^\omega$ to denote its cardinality, power set, set of finite words, and set of infinite words, respectively.
We define $\Sigma^\infty = \Sigma^* \cup \Sigma^\omega$ and denote the empty string by $\emptyset$.

\begin{definition}[\bf Transition System]
\label{def:ts}
A (weighted) transition system (TS) is a tuple $\BFT := (\CQ_\RMT, q_\RMT^0, \delta_\RMT, \Pi_\RMT, \CL_\RMT, w_\RMT)$, where
\begin{enumerate}
\item $\CQ_\RMT$ is a finite set of states; %
\item $q_\RMT^0 \in \CQ_\RMT$ is the initial state; %
\item $\delta_\RMT \subseteq \CQ_\RMT \times \CQ_\RMT$ is the transition relation; %
\item $\Pi_\RMT$ is a finite set of atomic propositions; %
\item $\CL_\RMT:\CQ_\RMT\to 2^{\Pi_\RMT}$ is a map giving the set of atomic propositions satisfied in a state; %
\item $w_\RMT: \delta_{T}\to \BBN_{>0}$ is a map that assigns a positive integer weight to each transition.
\end{enumerate}
\end{definition}

We define a run of $\BFT$ as an infinite sequence of states $r_\RMT = q^0,q^1,\ldots$ such that $q^0 = q^{0}_{T}$, $q^k \in \CQ_\RMT$ and $(q^k,q^{k+1}) \in \delta_\RMT$ for all $k \geq 0$.
A run generates an infinite word $\omega_\RMT = \CL(q^0),\CL(q^1),\ldots$ where $\CL(q^k)$ is the set of atomic propositions satisfied at state $q^k$.

In this work, we consider mission specifications expressed in Linear Temporal Logic (LTL)~\cite{CB-JPK:2008,EMC-DP-OG:1999}.
Informally, an LTL formula over the set $\Pi$ of atomic propositions may contain boolean operators $\notltl$ (negation), $\orltl$ (disjunction) and $\andltl$ (conjunction), and temporal operators $\Next$ (next), $\Until$ (until), $\Event$ (eventually) and $\Always$ (globally/always).
LTL formulas are interpreted over infinite words (generated by the transition system $\BFT$ from Def.~\ref{def:ts}).
For instance, $\Next\,\prop$ states that at the next position of a word, proposition $\prop$ is true.
The formula $\prop_1\,\Until\,\prop_2$ states that there is a future position of the word when proposition $\prop_2$ is true, and proposition $\prop_1$ is true at least until $\prop_2$ is true.
The formula $\Always\,\prop$ states that $\prop$ is true at all positions of the word; the formula $\Event\,\prop$ states that $\prop$ eventually becomes true in the word.
More expressivity can be achieved by combining the temporal and boolean operators.
We say a run $r_\RMT$ satisfies $\phi$ if and only if the word generated by $r_\RMT$ satisfies $\phi$.
An LTL formula $\phi$ over a set $\Pi$ can be represented by a \emph{\buchi automaton}, which is defined next.

\begin{definition}[\bf \buchi Automaton]
\label{def:buchi}
  A \buchi automaton is a tuple $\BFB \defeq (\CQ_\RMB,\CQ_\RMB^0,\Sigma_\RMB,$ $\delta_\RMB,\CF_\RMB)$, consisting of %
\begin{enumerate}
\item a finite set of states $\CQ_\RMB$; %
\item a set of initial states $\CQ_\RMB^0\subseteq \CQ_\RMB$; %
\item an input alphabet $\Sigma_\RMB$; %
\item a non-deterministic transition relation $\delta_\RMB \subseteq \CQ_\RMB\times \Sigma_\RMB \times \CQ_\RMB$; %
\item a set of accepting (final) states $\CF_\RMB\subseteq \CQ_\RMB$.
\end{enumerate}
\end{definition}

A \emph{run} of $\BFB$ over an input word $\omega=\omega^0,\omega^1,\ldots$ is a sequence $r_\RMB=q^0,q^1,\ldots$, such that $q^0 \in \CQ_\RMB^0$, and $( q^k,\omega^k,q^{k+1}) \in \delta_\RMB$, for all $k\geq 0$.
A \buchi automaton $\BFB$ accepts a word over $\Sigma_\RMB$ if and only if at least one of the corresponding runs intersects with $\CF_\RMB$ infinitely many times.
For any LTL formula $\phi$ over a set $\Pi$, one can construct a \buchi automaton with input alphabet $\Sigma_\RMB = 2^{\Pi}$ accepting all and only words over $2^\Pi$ that satisfy $\phi$.

\begin{definition}[\bf Prefix-Suffix Structure]
\label{def:prefix_suffix}
A prefix of a run is a finite path from an initial state to a state $q$.
A periodic suffix is an infinite run originating at the state $q$ reached by the prefix, and periodically repeating a finite path, which we call the suffix cycle, originating and ending at $q$.
A run is in prefix-suffix form if it consists of a prefix followed by a periodic suffix.
\end{definition}

\section{Problem Formulation and Approach}
\label{sec:prob}
In this section we introduce the multi-robot path planning problem with temporal constraints for robots with uncertain, but bounded traveling times.
Let
\begin{equation}
\label{eqn:graph}
\mathcal{E}=(V,\rightarrow_\mathcal{E})
\end{equation}
be a directed graph, where $V$ is the set of vertices and $\rightarrow_{\mathcal{E}}\subseteq V\times V$ is the set of edges.
We consider $\mathcal{E}$ as the quotient graph of a partitioned environment, where $V$ is the set of labels of the regions in the environment and $\rightarrow_{\mathcal{E}}$ is the corresponding adjacency relation.
For instance, $V$ can be a set of labels for the regions and intersections for a road network and $\rightarrow_\mathcal{E}$ can give their connections (see Fig.~\ref{fig:road_network}). 

Consider a team of $m$ robots moving in an environment modeled by $\mathcal E$.
The motion capabilities of robot $i,i=1,\ldots,m$ are modeled by a TS $\BFT_i = (\CQ_i,q_i^0,\delta_i,\Pi_i,$ $\CL_i,w_i)$, where $\CQ_{i}\subseteq V$; $q_i^0$ is the initial vertex of robot $i$; $\delta_i\subseteq \rightarrow_{\mathcal E}$ is a relation modeling the capability of robot $i$ to move among the vertices; $\Pi_i\subseteq\Pi$ is the subset of propositions that can be satisfied by robot $i$; $\CL_i$ is a mapping from $\CQ_i$ to $2^{\Pi_i}$ showing how the propositions are satisfied at vertices; and $w_i(q,q')$ gives the \emph{nominal} time for robot $i$ to go from vertex $q$ to $q'$, which we assume to be a positive integer.
However, due to the uncertainty in the traveling times of the robots, the \emph{actual} time it takes for robot $i$ to go from $q$ to $q'$, which we denote by $\tilde w_i(q,q')$, is a non-deterministic quantity that lies in the interval $[\underline{\rho_i}w_i(q,q'), \overline{\rho_i}w_i(q,q')]$, where $\underline{\rho_i},\overline{\rho_i}$ are the predetermined lower and upper \emph{deviation values} of robot $i$ that satisfy $0<\underline{\rho_i}\leq1\leq\overline{\rho_i}$.
In this model, robot $i$ travels along the edges of $\BFT_{i}$, and spends zero time at the vertices.
We also assume that the robots are equipped with motion primitives that allow them to deterministically move from $q$ to $q'$ for each $(q,q')\in \delta_i$, even though the time it takes to reach from $q$ to $q'$ is uncertain.
In the following, we use the expression ``\emph{in the field}'' to refer to the model with uncertain traveling times, and use $x$ and $\tilde x$ to denote the \emph{nominal} and \emph{actual} values of some variable $x$.

We consider the case where the robotic team has a mission in which some particular task must be repeatedly completed and the maximum time in between successive completions of this task must be minimized.
For instance, in a persistent data gathering mission, the global mission could be \emph{keep gathering data while obeying traffic rules at all times}, and the repeating task could be \emph{gathering data}.
For this example, the robots would operate according to the mission specification while ensuring that the maximum time in between any two successive data gatherings is minimized.
Consequently, we assume that there is an \emph{optimizing proposition} $\opt \in \Pi$ that corresponds to this particular repeating task and consider multi-robot missions specified by LTL formulae of the form
\begin{equation}
\label{eqn:general_formula}
\phi\defeq\varphi \land \Always\Event\opt,
\end{equation}
where $\varphi$ can be any LTL formula over $\Pi$, and $\Always\Event\opt$ means that the proposition $\opt$ must be repeatedly satisfied.
Our aim is to plan multi-robot paths that satisfy the mission $\phi$ and minimize the maximum time in between successive satisfying instances of $\opt$.

To state this problem formally, we assume that each run $r_i = q_i^0,q_i^1,\ldots$ of $\BFT_i$ (robot $i$) starts at $t=0$ and generates a word $\omega_i = \omega_i^0,\omega_i^1,\ldots$ and a corresponding sequence of time instances $\BBT_i \defeq t_i^0,t_i^1,\ldots$ such that $\omega_i^k = \CL_i(q_i^k)$ is satisfied at $t_i^k$.
To define the behavior of the team as a whole, we interpret the sequences $\BBT_{i}$ as sets and take the union $\bigcup_{i=1}^{m} \BBT_{i}$ and order this set in an ascending order to obtain the sequence $\BBT \defeq t^0,t^1,\ldots$.
Then, we define $\omega_{team} = \omega_{team}^0,\omega_{team}^1,\ldots$ to be the word generated by the team of robots where $\omega_{team}^k$ is the union of all propositions satisfied at $t^k$.
Finally, we define the infinite sequence $\BBT^\opt = \BBT^\opt(1),\BBT^\opt(2),\ldots$ where $\BBT^\opt(k)$ stands for the time instance when $\opt$ is satisfied for the $k^{th}$ time by the team and define the cost function 
\begin{equation}
\label{eqn:cost_function}
J(\BBT^{\opt})=\limsup_{i\to+\infty}\left(\mathbb{T}^{\opt}(i+1) - \mathbb{T}^{\opt}(i)\right).
\end{equation}
The form of the cost function given in Eq.~\eqref{eqn:cost_function} is motivated by persistent surveillance missions, where one is interested in the long-term behavior of the team.
Given a sequence $\BBT^\pi$ corresponding to some run of the team, the cost function in Eq.~\eqref{eqn:cost_function} captures the maximum time between satisfying instances of $\pi$ once the team behavior reaches a steady-state, which we achieve in finite time as we will discuss in Sec.~\ref{sec:sol.run}.
Thus, the problem becomes that of finding an optimal run of the team that satisfies $\phi$ and minimizes~\eqref{eqn:cost_function}.
However, the non-determinism in traveling times imposes two additional difficulties which directly follow from Prop.~3.2 in \cite{AU-SLS-XCD-CB:2011}:
First, if the traveling times observed during deployment deviate from those used in planning, then there exist missions that will be violated in the field.
Second, the worst case performance of the robotic team during deployment in terms of Eq.~\ref{eqn:cost_function} will be limited by that of a single member.

To guarantee correctness in the field, and limit the deviation of the performance of the team from the planned optimal run during deployment, we propose periodic synchronization of the robots.
Using this synchronization protocol, robots synchronize with each other according to pre-computed synchronization sequences $s_i, i=1,\ldots,m$ as they execute their runs $r_i, i=1,\ldots,m$ in the field.
We can now formulate the problem.
\begin{problem}
\label{prb:robust}
Given a team of $m$ robots modeled as transition systems $\BFT_{i}$, $i\!=\!1,\ldots,m$, and an LTL formula $\phi$ over $\Pi$ in the form \eqref{eqn:general_formula}, synthesize individual runs $r_i$ and synchronization sequences $s_i$ for each robot such that $\BBT^\opt$ minimizes the cost function \eqref{eqn:cost_function}, and $\tilde\omega_{team}$, \ie the word observed in the field, satisfies $\phi$.
\end{problem}
Note that our aim in Prob.~\ref{prb:robust} is to find a run that is optimal under nominal values while ensuring that $\phi$ is never violated in the field.
Since $\tilde{\BBT}^\pi$, \ie the sequence of instants at which $\opt$ is satisfied during deployment, is likely to be sub-optimal, we will also seek to bound the deviation from optimality in the field.
As we consider LTL formulas containing $\Always\Event \opt$, this optimization problem is always well-posed.

Our solution to Problem~\ref{prb:robust} can be outlined as follows:
\begin{enumerate}
\item We obtain the team transition system $\BFT$ that captures the joint asynchronous behavior of the members of the team (See Sec.~\ref{sec:sol.tts});
\item We find an optimal satisfying run $r^\star_{team}$ on $\BFT$ using the \optrun algorithm we previously developed in \cite{SLS-JT-CB-DR:2011} and obtain individual optimal runs $r^\star_i,i=1,\ldots,m$ (See Sec. \ref{sec:sol.run}); 
\item We generate the synchronization sequences $s_i,i=1,\ldots,m$ to guarantee correctness in the field and calculate an upper bound on the field value of the cost function \eqref{eqn:cost_function} (See Sec. \ref{sec:sol.sync}).
\end{enumerate}

\section{Problem Solution}
\label{sec:sol}
In this section, we describe each step of our solution to Prob.~\ref{prb:robust} in detail with the help of a simple illustrative example.
We present our experimental results in Sec.~\ref{sec:exp}.

\subsection{Obtaining the Team Transition System}
\label{sec:sol.tts}
In~\cite{AU-SLS-XCD-CB-DR:2011.b}, we showed that the joint asynchronous behavior of a robotic team modeled as $m$ transition systems $\BFT_i, i=1,\ldots,m$ (Def.~\ref{def:ts}) can be captured using a region automaton.
A region automaton, as given in the following definition from ~\cite{AU-SLS-XCD-CB:2011}, is a finite transition system that keeps track of the relative positions of the robots as they move asynchronously in the environment.

\begin{definition}[\bf Region Automaton]
\label{def:ra}
The region automaton $\BFR$ is a TS (Def.~\ref{def:ts}) $\BFR \defeq (\CQ_\RMR,q_\RMR^0,\delta_\RMR, \Pi_\RMR, \CL_\RMR, w_\RMR)$, where
 	\begin{enumerate}
 	\item $\CQ_\RMR$ is the set of states of the form $(q,r)$ such that
	\begin{enumerate}
	\item $q$ is a tuple of state pairs $(q_1q_1',\ldots,q_mq_m')$ where the $i^{th}$ element $q_iq_i'$ is a source-target state pair from $\CQ_i$ of $\BFT_i$ meaning robot $i$ is currently on its way from $q_i$ to $q_i'$, and
	\item $r$ is a tuple of clock values $(x_1,\ldots,x_m)$ where $x_i\in\BBN$ denotes the time elapsed since robot $i$ left state $q_i$.
	\end{enumerate}
 	\item $q_\RMR^0\subseteq\CQ_\RMR$ is the set of initial states with $r=(0,\ldots,0)$ and $q=(q_1^0q_1',\ldots,q_m^0q_m')$ such that $q_i^0$ is the initial state of $\BFT_i$ and $(q_i^0,q_i') \in \delta_i$.
 	\item $\delta_\RMR$ is the transition relation such that a transition from $(q,r)$ to $(q',r')$ where the $i^{th}$ state pair $q_iq_i'$ and the $i^{th}$ clock value $x_i$ in $(q,r)$ change to $q_i'q_i''$ and $x_i'$ in $(q',r')$ exists if and only if
	\begin{enumerate}
	\item $(q_i,q_i'),(q_i',q_i'') \in \delta_i$ for all changed state pairs,
	\item $w_i(q_i,q_i') - x_i$ of all changed state pairs are equal to each other and are strictly smaller than those of unchanged state pairs, and
	\item for all changed state pairs, the corresponding $x_i'$ in $r'$ becomes $x_i'=0$ and all other clock values in $r$ are incremented by $w_i(q_i,q_i')-x_i$ in $r'$.
	\end{enumerate}
	\item $\Pi_\RMR=\cup_{i=1}^m\Pi_i$ is the set of propositions;
	\item $\CL_\RMR:\CQ_\RMR\to 2^{\Pi_\RMR}$ is a map giving the set of atomic propositions satisfied in a state. For a state $(q,r)$, $\CL_\RMR((q,r)) = \cup_{i|x_i=0}\CL_i(q_i)$;
	\item $w_\RMR: \delta_{R}\to \BBN_{>0}$ is a map that assigns a positive integer weight to each transition such that $w_\RMR((q,r),(q',r'))=w_i(q_i,q_i')-x_i$ for each state pair that has changed from $q_iq_i'$ to $q_i'q_i''$ with a corresponding clock value of $x_i'=0$ in $r'$.
 	\end{enumerate}
\end{definition}

\newexample{1}{%
Fig.~\ref{fig:example} illustrates the TS's of two robots that are expected to satisfy the mission $\phi := \Always(\prop_1 \Implies \Next (\notltl \prop_1\;\Until\;\prop_3))\andltl\Always \Event \opt$, where $\Pi_1=\{\prop_1,\,\opt\}$, $ \Pi_2=\{\prop_2,\,\prop_3,\,\opt\}$, and $\Pi = \{\prop_1,\,\prop_2,\,\prop_3,\,\opt\}$.
The region automaton $\BFR$ that models the robots is given in Fig.~\ref{fig:ra}.
}

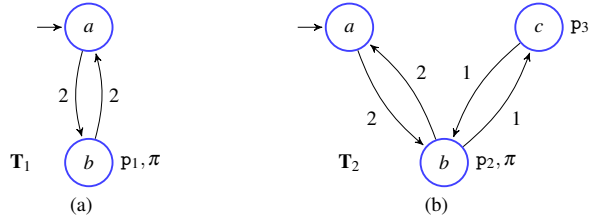
\begin{figure}[!ht]
\centering
\scalebox{0.9}{
	\subfigure[]{
		\label{fig:example.ts1}
		\begin{tikzpicture}[->,>=stealth',shorten >=1pt,auto,node distance=2cm,scale=2,initial text=,initial distance=0.2]

			\tikzstyle{every state}=[circle,thick,draw=blue!75,minimum size=7mm]
			\node [state,initial]						(a) at (0.5,1)	{$a$};
			\node [state,label=right:${\prop_1,\opt}$]	(b) at (0.5,0)	{$b$};
			\node []									(name) at (0,0) {$\BFT_1$};
			
			\path (a) edge [bend right=15,swap] node {2} (b);
			\path (b) edge [bend right=15,swap] node {2} (a);
	
		\end{tikzpicture}
	}%
} \hskip 0.5in%
\scalebox{0.9}{
	\subfigure[]{
		\label{fig:example.ts2}
		\begin{tikzpicture}[->,>=stealth',shorten >=1pt,auto,node distance=2cm,scale=2,initial text=,initial distance=0.2]

			\tikzstyle{every state}=[circle,thick,draw=blue!75,minimum size=7mm]
			\node [state,initial]						(a) at (0,1)	{$a$};
			\node [state,label=right:${\prop_2,\opt}$]	(b) at (0.7,0)	{$b$};
			\node [state,label=right:$\prop_3$]			(c) at (1.4,1)	{$c$};
			\node []									(name) at (0,0) {$\BFT_2$};
			
			\path (a) edge [bend right=15,swap] node {2} (b);
			\path (b) edge [bend right=15,swap] node {2} (a);
			\path (b) edge [bend right=15,swap] node {1} (c);
			\path (c) edge [bend right=15,swap] node {1} (b);
	
		\end{tikzpicture}
	}%
}%
\caption{Figs. (a) and (b) show the transition systems $\BFT_1$ and $\BFT_2$ of two robots in an environment with three vertices.
The states of the transition systems correspond to vertices $\{a,b,c\}$ and the edges represent the motion capabilities of each robot.
The weights of the edges represent the traveling times between any two vertices.
The propositions $\prop_1,\prop_2,\prop_3$ and $\opt$ are shown next to the vertices where they can be satisfied by the robots.}
\label{fig:example}
\end{figure}

~\\
However, as a region automaton encodes the directions of travel of the robots as opposed to their locations, it typically contains redundant states, and thus can typically be reduced to a smaller size.
The following example illustrates this fact.

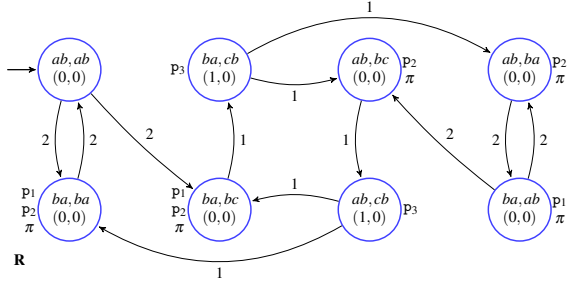
\begin{figure}[h]
\centering
\scalebox{0.7}{
	\begin{tikzpicture}[->,>=stealth',shorten >=1pt,auto,node distance=1.5cm,scale=1.9,initial distance=0.3]

		\tikzstyle{every state}=[circle,thick,draw=blue!75,minimum size=8mm, text centered]
		\node [state,initial,initial text=,initial left,text width=2.2em] (ab0ab0) at (1,1.4) {$ab,ab$ $(0,0)$};
		\node [state,text width=2.2em,label={left,xshift=0.25em}:\footnotesize$\prop_3$\normalsize] (ba1cb0) at (2.5,1.4) {$ba,cb$ $(1,0)$};
		\node [state,text width=2.2em,label={right,xshift=-0.5em}:\footnotesize$\begin{array}{c}\prop_2\\\pi\end{array}$\normalsize] (ab0bc0) at (4,1.4) {$ab,bc$ $(0,0)$};
		\node [state,text width=2.2em,label={right,xshift=-0.5em}:\footnotesize$\begin{array}{c}\prop_2\\\pi\end{array}$\normalsize] (ab0ba0) at (5.5,1.4) {$ab,ba$ $(0,0)$};
		\node [state,text width=2.2em,label={left,xshift=0.5em}:\footnotesize$\begin{array}{c}\prop_1\\\prop_2\\\pi\end{array}$\normalsize] (ba0ba0) at (1,0) {$ba,ba$ $(0,0)$};
		\node [state,text width=2.2em,label={left,xshift=0.5em}:\footnotesize$\begin{array}{c}\prop_1\\\prop_2\\\pi\end{array}$\normalsize] (ba0bc0) at (2.5,0) {$ba,bc$ $(0,0)$};
		\node [state,text width=2.2em,label={right,xshift=-0.25em}:\footnotesize$\prop_3$\normalsize] (ab1cb0) at (4,0) {$ab,cb$ $(1,0)$};
		\node [state,text width=2.2em,label={right,xshift=-0.5em}:\footnotesize$\begin{array}{c}\prop_1\\\pi\end{array}$\normalsize] (ba0ab0) at (5.5,0) {$ba,ab$ $(0,0)$};
		\node [] (name) at (0.5,-0.5) {$\BFR$};
		
		\path (ab0ab0) edge [bend right=15,swap] node {2} (ba0ba0);
		\path (ab0ab0) edge [bend right=5] node {2} (ba0bc0);
		\path (ba1cb0) edge [bend left=30] node {1} (ab0ba0);
		\path (ba1cb0) edge [bend right=15,swap] node {1} (ab0bc0);
		\path (ab0bc0) edge [bend right=15,swap] node {1} (ab1cb0);
		\path (ab0ba0) edge [bend right=15,swap] node {2} (ba0ab0);
		\path (ba0ab0) edge [bend right=15,swap] node {2} (ab0ba0);
		\path (ba0ab0) edge [bend left=5,swap] node {2} (ab0bc0);
		\path (ab1cb0) edge [bend right=15,swap] node {1} (ba0bc0);
		\path (ab1cb0) edge [bend left=30] node {1} (ba0ba0);
		\path (ba0bc0) edge [bend right=15,swap] node {1} (ba1cb0);
		\path (ba0ba0) edge [bend right=15,swap] node {2} (ab0ab0);

	\end{tikzpicture}
}\hfill
\sidecaption
\caption{The finite state region automaton capturing the joint behavior of two robots in 9 states.
In a circle representing a state $(q,r)$, the first line is $q$ and the second line is $r$. }
\label{fig:ra}
\end{figure}

~\\
\examplerevisited{1}{%
State $((ab,bc),(0,0))$ of the region automaton $\BFR$ given in Fig.~\ref{fig:ra} is equivalent to the state $((ab,ba),(0,0))$ in the sense that both robots satisfy the same propositions and the positions of both robots are the same at both states, \ie robot 1 is at $a$ and robot 2 is at $b$.
These two states differ only in the future direction of travel of the second robot, \ie robot 2 travels towards $c$ in the first state whereas it travels towards $a$ in the second state.
This information, however, is redundant as it can be obtained just by looking at the next state of the team in any given run.
}

~\\
Motivated by this observation, we define a binary relation $\CR$ to reduce the region automaton $\BFR$ to a smaller team transition system $\BFT$.

\begin{definition}[\bf Binary Relation $\CR$]
\label{def:binary_rel}
Binary relation $\CR = \{ ( s,t) | s \in \CQ_\RMR,$ $t \in \CQ_\RMT\}$ is a mapping between the states of $\BFR$ and $\BFT$ that maps a state $s = ((q_1q_1',\ldots,q_mq_m'),$ $(x_1,\ldots,x_m))$ in $\CQ_\RMR$ to a state $t = (t_1,\ldots,t_m)$ in $\CQ_\RMT$, where $t_i = q_i$ if $x_i=0$ and $t_i=q_iq_i'x_i$ if $x_i>0$.
Note that, $x_i=0$ for at least one $i\in\{1,\ldots,m\}$.
We refer to a state $t_i\in\CQ_\RMT$ of the form $q_iq_i'x_i$ as a traveling state as it captures the instant where robot $i$ has traveled from $q_i$ to $q_i'$ for $x_i$ time units.
\end{definition}

Given a region automaton $\BFR$, we can obtain the corresponding team transition system $\BFT$ using the binary relation $\CR$ and the following procedure.

\begin{Procedure}[\bf Obtaining $\BFT$ from $\BFR$]
\label{prc:t_from_r}
Using $\CR$ we construct the team transition system $\BFT$ from the region automaton $\BFR$ as follows:
\begin{enumerate}
\item For each $s \in \CQ_\RMR$ we define the corresponding $t \in \CQ_\RMT$ as given in Def.~\ref{def:binary_rel} such that $( s, t) \in \CR$.
\item We set $\CL_\RMT(t) = \CL_\RMR(s)$. Note that, each $s$ that corresponds to a given $t$ has the same set of propositions due to the way $\BFR$ is constructed (Def.~\ref{def:ra})~\cite{AU-SLS-XCD-CB-DR:2011.b}.
\item For each $s$ corresponding to a given $t$, we define the corresponding transitions originating from $t$ in $\BFT$ such that $\exists ( t,t')\in\delta_\RMT \, \forall \, ( s, s')\in\delta_\RMR$ where $( s,t)\in\CR$ and $( s',t')\in\CR$.
\item We mark a state $t$ in $\CQ_\RMT$ as the initial state of $\BFT$ if the corresponding $s$ is an initial state in $\CQ_\RMR$. Note that, all states that correspond to a given $t$ are either in $q^0_\RMR$ altogether or none of them are in $q^0_\RMR$.
\end{enumerate}
\end{Procedure}

The following proposition shows that the team transition system $\BFT$ obtained using Proc.~\ref{prc:t_from_r} and the corresponding region automaton $\BFR$ are bisimulation equivalent, \ie there exists a binary relation between the states and the transitions of $\BFR$ and $\BFT$ such that they behave in the same way \cite{CB-JPK:2008}.

\begin{proposition}[\bf Bisimulation Equivalence]
\label{prp:tts_sim_ra}
The team transition system $\BFT$ obtained using Proc.~\ref{prc:t_from_r} and the region automaton $\BFR$ are bisimulation equivalent, \ie $\BFR \sim \BFT$, and $\CR$ is a bisimulation relation for $\BFR$ and $\BFT$.
\end{proposition}
\begin{proof}
In the following, we use $Post(s)$ to denote the set of states that can be reached from state $s$ after taking a single transition out of $s$.
For any $( s,t)\in\CR$ where $s \in \CQ_\RMR$ and $t \in \CQ_\RMT$, it holds that $\CL(s) = \CL(t)$.
Furthermore, for any $( s,t)\in\CR$ it also holds by construction that $\forall s' \in Post(s), \exists t' \in Post(t) | ( s',t')\in\CR$ and $\forall t' \in Post(t), \exists s' \in Post(s) | ( s', t') \in \CR$.
Finally, we also have $\forall s \in q^0_\RMR$, $\exists t \in q^0_\RMT | ( s,t) \in \CR$ and $\forall t \in q^0_\RMT, \exists s \in q^0_\RMR | ( s,t) \in \CR$.
Therefore, $\BFR$ and $\BFT$ are bisimulation equivalent, \ie $\BFR\sim\BFT$, and $\CR$ is a bisimulation relation for $\BFR$ and $\BFT$.
\qedsymbol
\end{proof}

\examplerevisited{1}{
Using $\CR$ we construct $\BFT$ (Fig.~\ref{fig:tts}) that captures the joint asynchronous behavior of the team in 6 states whereas the corresponding region automaton $\BFR$ had 9 states.
A state labeled $( a,b )$ means robot 1 is at region $a$ and robot 2 is at region $b$, whereas a state labeled $( ba1,c )$ means robot 1 traveled from $b$ to $a$ for 1 time unit and robot 2 is at c.
}

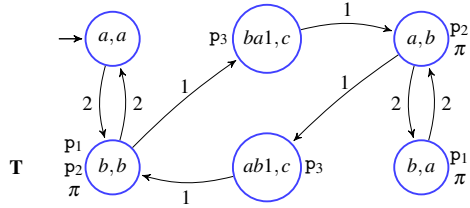
\begin{figure}[h]
\centering
\scalebox{0.9}{
	\begin{tikzpicture}[->,>=stealth',shorten >=1pt,auto,node distance=1.5cm,scale=1.9,initial distance=0.2]

		\tikzstyle{every state}=[circle,thick,draw=blue!75,minimum size=8mm, text centered]
		\node [state,initial,initial text=,initial left] (aa) at (0,1) {$a,a$};
		\node [state,label={left,xshift=0.2em}:\footnotesize$\prop_3$\normalsize] (ba1c) at (1.2,1) {$ba1,c$};
		\node [state,label={right,xshift=-0.5em}:\footnotesize$\begin{array}{c}\prop_2\\\pi\end{array}$\normalsize] (ab) at (2.4,1) {$a,b$};
		\node [state,label={left,xshift=0.5em}:\footnotesize$\begin{array}{c}\prop_1\\\prop_2\\\pi\end{array}$\normalsize] (bb) at (0,0) {$b,b$};
		\node [state,label={right,xshift=-0.2em}:\footnotesize$\prop_3$\normalsize] (ab1c) at (1.2,0) {$ab1,c$};
		\node [state,label={right,xshift=-0.5em}:\footnotesize$\begin{array}{c}\prop_1\\\pi\end{array}$\normalsize] (ba) at (2.4,0) {$b,a$};
		\node []									(name) at (-0.75,0) {$\BFT$};
		
		\path (aa) edge [bend right=15,swap] node {2} (bb);
		\path (ba1c) edge [bend left=15] node {1} (ab);
		\path (ab) edge [bend right=5,swap] node [xshift=0.7em] {1} (ab1c);
		\path (ab) edge [bend right=15,swap] node {2} (ba);
		\path (ba) edge [bend right=15,swap] node {2} (ab);
		\path (ab1c) edge [bend left=15] node {1} (bb);
		\path (bb) edge [bend right=15,swap] node {2} (aa);
		\path (bb) edge [bend left=5] node [xshift=0.7em] {1} (ba1c);

	\end{tikzpicture}
}\hfill
\sidecaption
\caption{The team transition system capturing the joint behavior of two robots in 6 states.}
\label{fig:tts}
\end{figure}

In~\cite{AU-SLS-XCD-CB-DR:2011.b} we showed that the number of states $|\CQ_\RMR|$ of the region automaton $\BFR$ that models the $m$ TSs $\BFT_i, i=1,\ldots,m$ is bounded by $\left(\prod_{i=1}^m|\delta_i|\right) \left(\prod_{i=1}^{m} W_i - \prod_{i=1}^{m} (W_i - 1) \right)$, where $|\delta_i|$ is the number of transitions in the TS $\BFT_i$ of robot $i$ and $W_i$ is maximum weight of any transition in $\BFT_i$.
The following proposition provides a bound on the number of states $|\CQ_\RMT|$ of $\BFT$ and shows that it is indeed significantly smaller than the bound on $|\CQ_\RMR|$.
\begin{proposition}
\label{prp:tts_state_bound}
The number of states $|\CQ_\RMT|$ of $\BFT$ is bounded by
\begin{equation}
\label{eqn:tts_state_bound}
\prod_{i=1}^m|\CQ_i|+(W-1)\prod_{i=1}^m|\delta_i|
\end{equation}
where $W$ is the largest edge weight in all TS's.
\end{proposition}
\begin{proof}
The first term in \eqref{eqn:tts_state_bound} is the maximum number of states that we can have in the Cartesian product of $T_i, i=1,\ldots,m$.
The second term in \eqref{eqn:tts_state_bound} is an upper-bound on the number of traveling states (Def.~\ref{def:binary_rel}) that we can define as we construct $\BFT$.
Here, $\prod_{i=1}^m|\delta_i|$ is the maximum number of different transition tuples that we can consider at line 8 of Alg.~\ref{alg:construct_tts} and $(W-1)$ is the upper bound on the number of new traveling states per transition tuple.
Thus, $|\CQ_\RMT|$ is bounded by the sum of these two terms as given in \eqref{eqn:tts_state_bound}.
\qedsymbol
\end{proof}

Finally, we note that the states of $\BFT$ corresponds to the instants where at least one member of the team has completed a transition in its individual TS and is currently at a vertex while other robots may still be traveling.
Using this fact, one can construct $\BFT$ directly by using a depth first search that runs in parallel on the TS's of the individual members of the team as given in Alg. ~\ref{alg:construct_tts}.

\begin{algorithm}[h]
\DontPrintSemicolon 
\SetInd{0.5em}{0.5em}
\KwIn{$(\BFT_1,\ldots,\BFT_m)$.}
\KwOut{Corresponding team transition system $\BFT$.}
\BlankLine
$q^0_\RMT := ( q^0_1,\ldots,q^0_m)$, where $q^0_i$ is the initial state of $\BFT_i$.\;
{\bf dfsT}($q^0_\RMT$).\;
\BlankLine
	\hrule	
	{\bf Function dfsT}(state tuple $q \in \CQ_\RMT$)
	\hrule
	\BlankLine
		$q[i]$ is the $i^{th}$ element of state tuple $q\in\CQ_\RMT$.\;
		$t_i$ is a transition of $\BFT_i,i=1,\ldots,m$, such that $t_i \in \{( q[i],q_i' )|( q[i],q_i')\in\delta_i\}$ if $q[i] \in \CQ_i$. Else if $q[i] = q_iq_i'x_i$, then $t_i = ( q_i,q_i')$.\;
		$T := ( t_1,\ldots,t_m)$ is a tuple of such transitions.\;
		$\CT$ is the set of all such transition tuples at $q$.\;
		\ForEach {transition tuple $T \in \CT$} {
		$w \leftarrow$ Shortest time until a robot is at a vertex while the transitions in $T$ are being taken.\;
			Find the $q'$ that corresponds to this new state of the team using $\CR$.\;
			\If{$q' \notin \CQ_\RMT$} {
				Add state $q'$ to $\CQ_\RMT$.\;
				Set $\CL_\RMT(q') = \cup_{i|q'[i]\in\CQ_i} \CL_i(q'[i])$.\;
				Add $( q,q')$ to $\delta_\RMT$ with weight $w$.\;
				Continue search from $q'$: {\bf dfsT}$(q')$.\;
			}
			\ElseIf{$( q,q') \notin \delta_\RMT$} {
				Add $( q,q')$ to $\delta_\RMT$ with weight $w$.\;
			}
		} 
\caption{\sc Construct-Team-TS}\label{alg:construct_tts}
\end{algorithm}

Alg.~\ref{alg:construct_tts} is essentially a recursive depth first search (lines 4 -- 17) that starts at the initial state of the team transition system $\BFT$ (line 3).
The initial state $q^0_\RMT$ of $\BFT$ is defined as the tuple of the initial states of the $m$ $\BFT_i$s (line 2).
Given a state $q$ of $\BFT$, the function \texttt{dfsT} first generates all possible tuples of transitions that can be taken at the current states of the $m$ TSs (lines 4 -- 7).
The current state of TS $\BFT_i$ is given by the $i^{th}$ element $q[i]$ of the current state $q$ of the $\BFT$.
At line 5 of Alg.~\ref{alg:construct_tts}, we consider all possible transitions out of the current states of all TSs $\BFT_i,i=1,\ldots,m$.
If $q[i]\in\CQ_i$, \ie $q[i]$ is a regular state of $\BFT_i$, then all transitions going out of this state in $\BFT_i$ will be considered in the transition tuples that we will construct.
Else, $q[i]$ is a traveling state of $\BFT_i$ of the form $q_iq_i'x_i$, and the only transition that can be taken is the one that is being taken, \ie the transition from $q_i$ to $q_i'$.
Then, we construct the set of all possible tuples of transitions that can be taken at the current states of the $m$ TSs (lines 6--7) and process each tuple one by one (lines 8--17).
In a transition tuple $T$, the $i^{th}$ element gives the transition that can be taken at the current state of $\BFT_i$.
In lines 9--10, we find the next instant where at least one transition in $T$ has been completed and the next state $q'$ of $\BFT$ that has been reached.
If $q'$ is a new state (lines 11 -- 15), we accordingly add it to $\CQ_\RMT$ and define its propositions.
Then, we add the transition that has just been completed to $\delta_\RMT$ and continue our search from this new state $q'$.
Else, we add the transition that has just been completed to $\delta_\RMT$ if required and proceed to the next transition tuple in $\CT$.
The algorithm concludes when all states and transitions of $\BFT$ have been discovered.

~\\
\myremark{1}{Comparison with Naive Construction}{%
One can avoid going through Alg.~\ref{alg:construct_tts} and capture the joint behavior of the team by discretizing each transition in $\BFT_i,i=1,\ldots,m$ to unit-length edges and taking the synchronous product of these $m$ $\BFT_i$'s.
This approach, however, yields a much larger model whose state count is bounded by 
\[
\prod_{i=1}^{m}\left(|\CQ_i|+\sum_{( q,q') \in \delta_i}w_i(q,q')-|\delta_i|\right).
\]
For the case where we have $m$ identical robots in an environment with $Q$ vertices, $\Delta$ edges and a largest edge weight of $W$, the above given bound is $O((Q+\Delta W)^m)$, whereas the bound given by Prop.~\ref{prp:tts_state_bound} is $O(Q^m+\Delta^mW)$.
}

\subsection{Obtaining Optimal Satisfying Runs and Transition Systems with Traveling States}
\label{sec:sol.run}
After constructing $\BFT$ that models the team, we use Alg. \optrun from \cite{SLS-JT-CB-DR:2011} to obtain an optimal run $r^\star_{team}$ on $\BFT$ that minimizes the cost function \eqref{eqn:cost_function}.
The optimal run $r^\star_{team}$ is always in prefix-suffix form, consisting of a finite sequence of states of $\BFT$ (prefix), followed by infinite repetitions of another finite sequence of states of $\BFT$ (suffix) as given in Def.~\ref{def:prefix_suffix}.

~\\
\examplerevisited{1}{
For the example we have shown, running Alg. \optrun \cite{SLS-JT-CB-DR:2011} on $\BFT$ given in Fig.~\ref{fig:tts} for the formula $\phi = \Always(\prop_1 \Implies \Next (\notltl \prop_1\;\Until\;\prop_3))\andltl\Always \Event \opt$ results in the optimal run
\begin{center}
\scalebox{0.9}{%
\begin{tabular}{ c |c c c c c c c }
$\BBT$& 0 & 2 & 3 & 4 & 5 & 6 & \ldots \\
\hline
$r^\star_{team}$ & a,a& b,b& ba1,c& a,b& ab1,c& b,b& \ldots\\
\hline
$\CL_\BFT(\cdot)$ & $\emptyset$ & $\prop_1,\prop_2,\opt$ & $\prop_3$ & $\prop_2,\opt$ & $\prop_3$ & $\prop_1,\prop_2,\opt$ & \ldots \\
\end{tabular}
}
\end{center}
where the first row shows when transitions occur, the second row corresponds the run $r_{team}^\star$, and the last row shows the satisfying atomic propositions.
For this run, $(a,a),(b,b)$ is the finite prefix and $(ba1,c),(a,b),(ab1,c),(b,b)$ is the suffix cycle, which will be repeated infinite number of times.
Also, the time sequence $\BBT^\opt$ of satisfaction of $\opt$ is $\BBT^\opt = 2,4,6,8,\ldots $ and the cost as defined in \eqref{eqn:cost_function} is $J(\BBT^\opt) = 2$.
}

~\\
Since $\BFT$ captures the asynchronous motion of the robots, the optimal satisfying run $r^\star_{team}$ on $\BFT$ may contain some traveling states which do not appear in the individual TSs $\BFT_i, i=1,\ldots,m$ that we started with.
But we cannot ignore such traveling states either, as each one of them is a candidate synchronization point for the corresponding robot as we discuss in Sec.~\ref{sec:sol.sync}.
Instead, we insert those traveling states into the individual TSs so that the robots will be able to synchronize with each other at those points if needed.
In the following, we use $q^k[i]$ to denote the $i^{th}$ element of the $k^{th}$ state tuple in $r^\star_{team}$, which is also the state of robot $i$ at that position of $r^\star_{team}$.
As given in Def.~\ref{def:binary_rel}, a traveling state of robot $i$ has the form $q_iq_i'x_i$.
First, we construct the set $\CS=\{(i,q^k[i])\;|\;q^k[i]=q_iq_i'x_i\,\forall\,k,i\}$ of all traveling states that appear in $r^\star_{team}$.
Elements of $\CS$ are tuples where the second element is a traveling state and the first element gives the transition system this new traveling state will be added to.
Next, we construct the set $\CT=\{(i,(q^k[i],q^{k+1}[i]),x)\;|\; ((i,q^k[i])\in\CS)\orltl((i,q^{k+1}[i])\in\CS),\, x= w_\RMT(q^k,q^{k+1})\,\forall\,k,i\}$ of all transitions that involve any of the traveling states in $r^\star_{team}$.
Elements of $\CT$ are triplets where the second element is a transition, the third element is the weight of this transition, and the first element shows the transition system that this new transition will be added to.
Then, we add the traveling states in $\CS$ and the transitions in $\CT$ to their corresponding transition systems.
Finally, using the following definition, we project the optimal satisfying run $r^\star_{team}$ down to individual robots $\BFT_i, i=1,\ldots,m$ to obtain individual optimal satisfying runs $r^\star_i,i=1,\ldots,m$.

\begin{definition}[\bf Projection of a Run on $\BFT$ to $\BFT_i$'s]
\label{def:projection_of_runs}
	Given a run $r_{team}$ on $\BFT$ where $r_{team}=q^0,q^1,\ldots$, we define its projection on $\BFT_i$ as run $r_i=q_i^0,q_i^1,\ldots$ for all $i=1,\ldots,m$, such that $q_i^k = q^k[i]$ where $q^k[i]$ is the $i^{th}$ element of tuple $q^k$.
\end{definition}

It can be easily seen that the set of runs $r_i,i=1,\ldots,m$ obtained from $r_{team}$ using Def.~\ref{def:projection_of_runs} and the run $r_{team}$ on $\BFT$ indeed correspond to each other:
The projection given in Def.~\ref{def:projection_of_runs} simply breaks down a sequence of tuples of states into a tuple of sequences of states, while preserving the order of the states.
Thus, the word $\omega$ and the time sequence $\BBT$ generated by $r_i,i=1,\ldots,m$ are exactly the word $\omega_{team}$ and the time sequence $\BBT_{team}$ generated by $r_{team}$.
Moreover, if run $r_{team}$ is in prefix-suffix form, all individual runs $r_i$ projected from $r_{team}$ are also in prefix-suffix form.
Therefore, the individual runs projected from the optimal run $r^\star_{team}$ are always in prefix-suffix form.

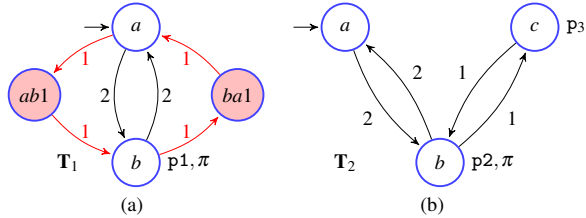
\begin{figure}[!ht]
\centering
\scalebox{0.9}{
\subfigure[]{
	\label{fig:comp_ts1}
	\begin{tikzpicture}[->,>=stealth',shorten >=1pt,auto,node distance=2cm,scale=2,initial text=,initial distance=0.2]

		\tikzstyle{every state}=[circle,thick,draw=blue!75,minimum size=7mm]
		\node [state,initial] (a) at (0.5,1) {$a$};
		\node [state,fill=red!25] (ab1) at (-0.25, 0.5) {$ab1$};
		\node [state,label=right:$\mathtt{p1,\pi}$]	(b) at (0.5,0)	{$b$};
		\node [state,fill=red!25] (ba1) at (1.25, 0.5) {$ba1$};
		\node []									(name) at (0,0) {$\BFT_1$};
		
		\path (a) edge [bend right=25,swap] node [xshift=0.15em] {2} (b);
		\path (b) edge [bend right=25,swap] node [xshift=-0.15em] {2} (a);
		\path [red] (a) edge [bend right=15] node [xshift=-0.3em,yshift=0.3em] {1} (ab1);
		\path [red] (ab1) edge [bend right=15] node [xshift=-0.3em,yshift=-0.3em] {1} (b);
		\path [red] (b) edge [bend right=15] node [xshift=0.3em,yshift=-0.3em] {1} (ba1);
		\path [red] (ba1) edge [bend right=15] node [xshift=0.3em,yshift=0.3em] {1} (a);

	\end{tikzpicture}
}%
}%
\scalebox{0.9}{
\subfigure[]{
	\label{fig:comp_ts2}
	\begin{tikzpicture}[->,>=stealth',shorten >=1pt,auto,node distance=2cm,scale=2,initial text=,initial distance=0.2]

		\tikzstyle{every state}=[circle,thick,draw=blue!75,minimum size=7mm]
		\node [state,initial] (a) at (0,1)	{$a$};
		\node [state,label=right:$\mathtt{p2,\pi}$]	(b) at (0.7,0)	{$b$};
		\node [state,label=right:$\prop_3$]	(c) at (1.4,1)	{$c$};
		\node []									(name) at (0,0) {$\BFT_2$};
		
		\path (a) edge [bend right=15,swap] node {2} (b);
		\path (b) edge [bend right=15,swap] node {2} (a);
		\path (b) edge [bend right=15,swap] node {1} (c);
		\path (c) edge [bend right=15,swap] node {1} (b);

	\end{tikzpicture}
}%
}\hfill
\sidecaption
\caption{Figs. (a) and (b) show the TSs with new traveling states that correspond to the optimal run $r^\star_{team}$ that we computed for Ex.~1.
The new traveling states and transitions of $\BFT_1$ are highlighted in red.}
\label{fig:complemented_ts's}
\end{figure}

\examplerevisited{1}{%
For this example, we have $\CS=\{(1,ab1),(1,ba1)\}$ and $\CT=\{(1,(a,ab1),1),(1,(ab1,b),1),(1,(b,ba1),1),(1,(ba1,a),1)\}$.
Fig.~\ref{fig:complemented_ts's} illustrates the corresponding TSs with new traveling states and transitions highlighted in red.
Using Def.~\ref{def:projection_of_runs}, we obtain the runs of the individual robots as $r^\star_1 = a, b, ba1, a, ab1,$ $b, ba1, a, ab1, \ldots$ and $r^\star_2 = a, b, c ,b, c, b, c, b, c, \ldots$.
}

\subsection{Guaranteeing Correctness through Synchronization and the Optimality Bound}
\label{sec:sol.sync}

As the robots execute their infinite runs in the field, they synchronize with each other according to the synchronization sequences that we generate using Alg.~\ref{alg:sync_seq}.
The synchronization sequence $s_i$ of robot $i$ is an infinite sequence of pairs of sets.
The $k^{th}$ element of $s_i$, denoted by $s_i^k$, corresponds to the $k^{th}$ element $q_i^k$ of $r_i^\star$.
Each $s_i^k$ is a tuple of two sets of robots: $s_i^k = (s_{i,wait}^k, s_{i,notify}^k)$, where $s_{i,wait}^k$ and $s_{i,notify}^k$ are the \emph{wait-set} and \emph{notify-set} of $s_i^k$, respectively.
The \emph{wait-set} of $s_i^k$ is the set of robots that robot $i$ must wait for at state $q_i^k$ before satisfying its propositions and proceeding to the next state $q_i^{k+1}$ in $r_i^\star$.
The \emph{notify-set} of $s_i^k$ is the set of robots that robot $i$ must notify as soon as it reaches state $q_i^k$.
As we discussed earlier in Sec.~\ref{sec:sol.run}, the optimal run $r^\star_{team}$ of the team and the individual optimal runs $r^\star_i, i=1,\ldots,m$ of the robots are always in prefix-suffix form (Def.~\ref{def:prefix_suffix}).
Consequently, individual synchronization sequences $s_i$ of the robots are also in prefix-suffix form.

\begin{algorithm} 
\DontPrintSemicolon 
\SetInd{0.5em}{0.5em}
\KwIn{Individual optimal runs of the robots $\{r_1^\star,\dots,r_m^\star\}$, \buchi automaton $\BFB_{\notltl\phi}$ that corresponds to $\notltl\phi$.}
\KwOut{Synchronization sequence for each robot $\{s_1,\dots,s_m\}$.}
\BlankLine
$\CI = \{1,\ldots,m\}$.\;
$beg \leftarrow$ beginning of suffix cycle.\;
$end \leftarrow$ end of suffix cycle.\;
Initialize each $s_i$ so that all robots wait for and notify each other at every position of their runs.\;
\ForEach{$k = 0,\ldots,end$}{
	\ForEach{$i \in \CI$} {
		\If{$k\neq 0$ and $k\neq beg$}{
			\ForEach{$j \in \CI\setminus i$}{
				Remove $j$ from $s_{i,wait}^k$.\;
				Remove $i$ from $s_{j,notify}^k$.\;
				Construct the TS $\BFW$ that generates every possible $\tilde{\omega}_{team}$.\;
				\If{the language of $\BFB_{\notltl\phi}\times\BFW$ is not empty}{
					Add $j$ back to $s_{i,wait}^k$.\;
					Add $i$ back to $s_{j,notify}^k$.\;
				}
			}
		}
	}
}
Rest of each $s_i$ is an infinite repetition of its suffix-cycle, i.e. $s_i^{beg},\ldots,s_i^{end}$.\;
\caption{{\sc Sync-Seq}\label{alg:sync_seq}}
\end{algorithm}

Alg.~\ref{alg:sync_seq} is essentially a loop (lines 5 -- 14) that computes wait-sets and notify-sets for each position of the runs of the robots to guarantee correctness in the field.
Initially, synchronization sequences are set so that the robots wait for and notify all other robots at every position of their runs (line 4).
At line 7 of Alg.~\ref{alg:sync_seq}, if $k$ is the first position of the runs, we do not modify this initial value of $s_i^k$.
This ensures that all robots start executing their runs in a synchronized way.
Also, if $k$ is the beginning of the suffix cycle, we again keep this initial value of $s_i^k$ so that all robots synchronize with each other globally at the beginning of each suffix cycle.
This lets us define a bound on optimality, \ie the value of the cost function \eqref{eqn:cost_function} observed in the field, as given in Prop.~\ref{prp:bound_on_opt}.
For all other positions of the runs, we try to shrink the wait-set and notify-set of each $s_i^k$ so that communication effort is reduced while we can still guarantee correctness in the field (lines 9 -- 14).
To this end, we consider each one of the robots in robot $i$'s $k^{th}$ wait-set, \ie $s_{i,wait}^k$, one by one.
After removing some robot $j$ from the $s_{i,wait}^k$, we also remove robot $i$ from $s_{j,notify}^k$ accordingly (lines 10--12).
Then, given the $\underline{\rho_i}$ and $\overline{\rho_i}$ values of the robots, we construct the TS $\BFW$ that generates all possible words $\tilde{\omega}_{team}$ that can be observed in the field due to the uncertainties in the traveling times.
Next, we check if the language of the product $\BFB_{\notltl\phi}\times\BFW$ is empty or not, where $\BFB_{\notltl\phi}$ is the \buchi automaton corresponding to the negation of the LTL formula $\phi$ (line 12).
If the language of the product is empty, then robot $i$ indeed does not need to wait for robot $j$ at the $k^{th}$ position of its run.
Thus, we keep the new values of $s_{i,wait}^k$ and $s_{j,notify}^k$.
Else, we restore $s_{i,wait}^k$ and notify-set of $s_{j,notify}^k$ to their previous values (lines 13--14) and proceed with the next robot in $s_{i,wait}^k$.
Once every robot in $s_{i,wait}^k$ is considered, we proceed with the next robot in the team, and eventually next position of the run.
Notice that, the synchronization sequences generated by Alg.~\ref{alg:sync_seq} are free from any dead-locks as lines 9 -- 10, and lines 13 -- 14 ensure that if some robot $i$ waits for robot $j$ at position $k$, then robot $j$ notifies robot $i$ at position $k$, \ie $j\in s_{i,wait}^k \iff i\in s_{j,notify}^k\;\forall\;i,j,k$.
As the synchronization sequences of the robots are in prefix-suffix form and the robots synchronize with each other globally at the beginning of each suffix (line 8), at line 15, we define the rest of each synchronization sequence as an infinite repetition of its first suffix-cycle that we have just generated.
For a prefix of length $p$ and a suffix cycle of length $s$, the complexity of Alg.~\ref{alg:sync_seq} is $O((p+s)m^2 L)$ where $m$ is the number of robots and $L$ is the complexity of constructing $W\times B_{\notltl\phi}$ and checking emptiness of its language at each iteration.
The synchronization protocol that the robots follow in the field is given in Alg.~\ref{alg:sync_run}.

\begin{algorithm} 
\DontPrintSemicolon 
\SetInd{0.5em}{0.5em}
\KwIn{The run $r_i$ and synchronization sequence $s_i$ of robot $i$ .}
\BlankLine
$k \leftarrow 0$.\;
\While{True} {
	{\bf Notify} all robots in $s_{i,notify}^k$.\;
	{\bf Wait} until notification messages of all robots in $s_{i,wait}^k$ are received.\;
	Make transition to $r_i^{k+1}$ after satisfying the propositions at $r_i^k$.\;
	$k \leftarrow k+1$.\;
}
\caption{{\sc Sync-Run}\label{alg:sync_run}}
\end{algorithm}

The following proposition slightly extends the result of Prop. 4.5 in \cite{AU-SLS-XCD-CB:2011} by considering unequal lower and upper deviation values.
\begin{proposition}
\label{prp:bound_on_opt}
Suppose that each robot's deviation values are bounded by $\underline{\rho}$ and $\overline{\rho}$ where $\overline{\rho}\geq1\geq \underline{\rho}>0$ (i.e., $\underline{\rho_i} \geq \underline{\rho}$ and $\overline{\rho_i}\leq\overline{\rho}$ for each robot $i$).
Let $J(\BBT^\opt)$ be the cost of the planned robot paths and let $J(\tilde\BBT^\opt)$ be the actual value of the cost observed during deployment.
Then, if the robots follow the protocol given in Alg.~\ref{alg:sync_run} the field value of the cost satisfies
\[
J(\tilde\BBT^\opt) \leq J(\BBT^\opt)\overline{\rho} + d_s(\overline{\rho}-\underline{\rho})
\]
where $d_s$ is the planned duration of the suffix cycle.
\end{proposition}

~\\
\examplerevisited{1}{%
For the example we have shown throughout this section, we obtain the following individual optimal runs and synchronization sequences.
\begin{center}
\scalebox{0.9}{
\begin{tabular}{c | c c c c c c c}
$\BBT$& 0 & 2 & 3 & 4 & 5& 6 & \ldots \\
\hline
\hline
$r_1^\star$ & a               & b             & ba1             & a             & ab1           & b             & \ldots \\
$s_1$       & $(\{2\},\{2\})$ & $(\{\},\{\})$ & $(\{2\},\{2\})$ & $(\{\},\{\})$ & $(\{\},\{\})$ & $(\{\},\{\})$ & \ldots \\
$\CL_1(.)$  &                 & $\prop_1,\pi$ &                 &               &               & $\prop_1,\pi$ & \ldots \\
\hline
$r_2^\star$ & a               & b             & c               & b             & c             & b             & \ldots \\
$s_2$       & $(\{1\},\{1\})$ & $(\{\},\{\})$ & $(\{1\},\{1\})$ & $(\{\},\{\})$ & $(\{\},\{\})$ & $(\{\},\{\})$ & \ldots \\
$\CL_2(.)$  &                 & $\prop_2,\pi$ & $\prop_3$       & $\prop_2,\pi$ & $\prop_3$     & $\prop_2,\pi$ & \ldots
\end{tabular}
}
\end{center}
In a line corresponding to a synchronization sequence $s_i$, first and second elements of the tuple at position $k$ are $s_{i,wait}^k$ and $s_{i,notify}^k$, respectively.
}

~\\
We finally summarize our approach in Alg.~\ref{alg:multi_opt_run_robust} and show that this algorithm indeed solves Prob. \ref{prb:robust}.
We discuss the complexity of our approach in Rem.~2.
\begin{proposition}
\label{prp:final}
Alg.~\ref{alg:multi_opt_run_robust} solves Prob.~\ref{prb:robust}.
\end{proposition}
\begin{proof}
Note that Alg.~\ref{alg:multi_opt_run_robust} combines all steps outlined in this section.
The planned word $\omega_{team}$ generated by the entire team satisfies $\phi$, and minimizes \eqref{eqn:cost_function}, as shown in \cite{SLS-JT-CB-DR:2011}.
The synchronization sequences guarantee correctness in the field by ensuring that the $\tilde{\omega}_{team}$ generated in the field never violates $\phi$ for given deviation values.
Therefore, $\{r^{\star}_1,\ldots, r^{\star}_{m}\}$ and $\{s_1,\ldots,s_m\}$ as obtained from Alg.~\ref{alg:multi_opt_run_robust} is a solution to Prob.~\ref{prb:robust}.
\qedsymbol
\end{proof}

\myremark{2}{Computational Complexity}{
The main drawback of our approach is its computational complexity, which is exponential in the number of robots (due to generation of the team transition system and the synchronization sequences) and in the length of the LTL formula (due to the conversion to a \buchi automaton).
This cost, however, is justified by the globally optimal runs that our approach computes, and in practice, we can solve fairly large problems.
}
~\\

\begin{algorithm} 
\DontPrintSemicolon 
\SetInd{0.5em}{0.5em}
\KwIn{$m$ $\BFT_i$'s, corresponding deviation values, and a global LTL specification $\phi$ of the form \eqref{eqn:general_formula}.}
\KwOut{A set of optimal runs $\{r^{\star}_1,\ldots, r^{\star}_{m}\}$ that satisfies $\phi$ and minimizes \eqref{eqn:cost_function}, a set of synchronization sequences $\{s_1,\ldots,s_m\}$ that guarantees correctness in the field, and the bound on the performance of the team in the field.}
\BlankLine
Construct the team transition system $\BFT$ using Alg.~\ref{alg:construct_tts}.\;
Find an optimal run $r^\star_{team}$ on $\BFT$ using \optrun~\cite{SLS-JT-CB-DR:2011}.\;
Insert new traveling states to TSs according to $r^\star_{team}$ (See. Sec.~\ref{sec:sol.run}).\;
Obtain individual runs $\{r^\star_1,\ldots,r^\star_m\}$ using Def.~\ref{def:projection_of_runs}.\;
Generate synchronization sequences $\{s_1,\ldots,s_m\}$ using {\sc Sync-Seq} (Alg.~\ref{alg:sync_seq}).\;
Find the bound on optimality as given in Prop.~\ref{prp:bound_on_opt}.\;
\caption{{\sc Robust-Multi-Robot-Optimal-Run}\label{alg:multi_opt_run_robust}}
\end{algorithm}

\section{Implementation and Case-Study}
\label{sec:exp}
We implemented Alg.~\ref{alg:multi_opt_run_robust} as a python module (available at {\footnotesize \url{http://hyness.bu.edu/lomap/}}) and used it to plan optimal satisfying paths and synchronization sequences for the scenario that we consider in this section.
Our experimental platform (Fig.~\ref{fig:road_network}) is a road network comprising roads, intersections and task locations.
Fig.~\ref{fig:model} illustrates the model that captures the motion of the robots on this platform, where 1 time unit corresponds to 1.574 seconds.

\begin{figure}
\subfigure[]{
	\includegraphics[scale=0.52]{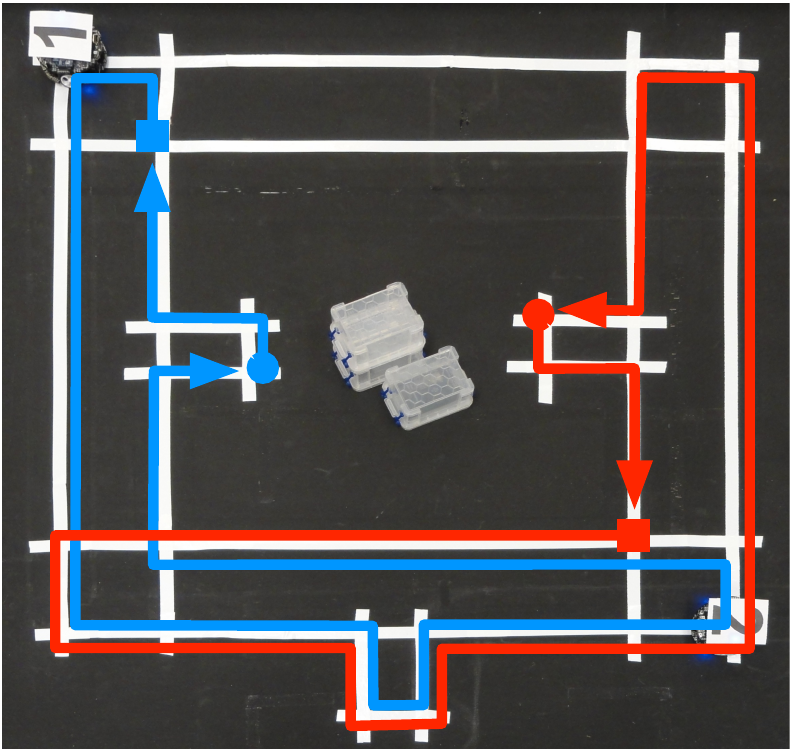}
	\label{fig:road_network}
}%
\subfigure[]{
	\includegraphics[scale=0.62]{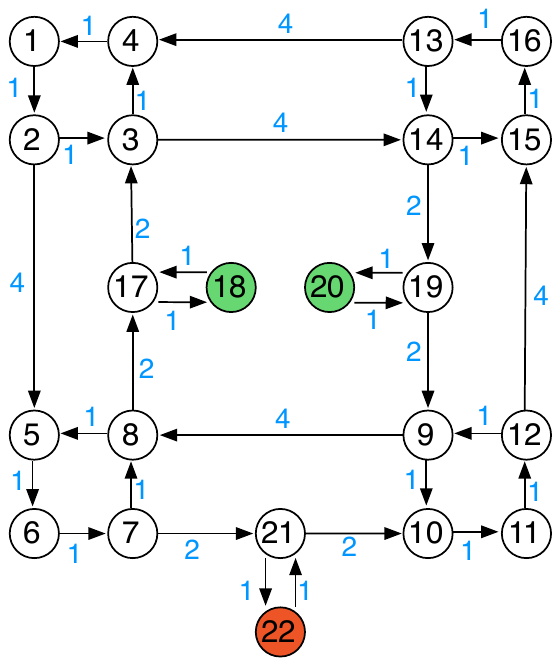}
	\label{fig:model}
}%
\sidecaption
\caption{Fig. (a) shows our experimental platform. The squares and the circles on the trajectories of the robots represent the beginning of the suffix cycle and sync. points, respectively. Fig. (b) illustrates the TS that models the robots. The green and red regions are data gather and upload locations, respectively.}
\label{fig:road_network_and_model}
\end{figure}

In our experiments, we consider a persistent surveillance task involving two robots with deviation values $\overline{\rho_1} = \overline{\rho_2} = 1.05$ and $\underline{\rho_1} = \underline{\rho_2} = 0.95$.
The building in the middle of the platform in Fig.~\ref{fig:road_network} is our surveillance target.
We define the set of propositions $\Pi=\{\mathtt{R1Gather18},$ $\mathtt{R1Gather20},$ $\mathtt{R2Gather18},$ $\mathtt{R2Gather20},$ $\mathtt{R1Gather},$ $\mathtt{R2Gather},$ $\mathtt{R1Upload},$ $\mathtt{R2Upload},\;$ $\mathtt{Gather}\}$ and assign them as $\CL_1(18)$ $= \{\mathtt{R1Gather18},\;$ $\mathtt{R1Gather},\;$ $\mathtt{Gather}\}$, $\CL_2(18)=\{\mathtt{R2Gather18},$ $\mathtt{R2Gather},$ $\mathtt{Gather}\}$, $\CL_1(20)=\{\mathtt{R1Gather20},$ $\mathtt{R1Gather},$ $\mathtt{Gather}\}$, $\CL_2(20)=\{\mathtt{R2Gather20},$ $\mathtt{R2Gather},$ $\mathtt{Gather}\}$, $\CL_1(22)=\{\mathtt{R1Upload}\}$ and $\CL_2(22)=\{\mathtt{R2Upload}\}$.
The main objective is to keep gathering data while minimizing the maximum time between successive gathers.
We require the robots to gather data in a synchronous manner at data gather locations 18 and 20 while ensuring that they do not gather data at the same place at the same time.
We also require the robots to upload their data at upload location 22 before their next data gather.
We express these requirements in LTL in the form of \eqref{eqn:general_formula} as 
\[
\begin{aligned}
    \phi = & \Always\,(\mathtt{R1gather} \Implies \Next(\notltl\mathtt{R1gather}\;\Until\;\mathtt{R1upload})) \andltl \Always\,(\mathtt{R2gather} \Implies\\
    &\Next(\notltl\mathtt{R2gather}\;\Until\;\mathtt{R2upload})) \andltl \Always\,((\mathtt{R1Gather18}\Implies \mathtt{R2Gather20}) \andltl\\
    &(\mathtt{R1gather20} \Implies \mathtt{R2gather18}) \andltl (\mathtt{R2gather18} \Implies \mathtt{R1gather20}) \andltl\\
    &(\mathtt{R2gather20} \Implies \mathtt{R1gather18})) \andltl \Always\Event\,\mathtt{Gather},
\end{aligned}
\]
where $\mathtt{Gather}$ is set as the optimizing proposition.

Fig.~\ref{fig:road_network} illustrates the solution we obtain using our algorithm.
Using an iMac i5 quad-core computer, it took our implementation 10 minutes to compute the optimal runs and synchronization sequences of the robots.
The planned value of the cost function was 44.072 seconds (28 time units) with an upper bound of 50.683 seconds (32.2 time units) seconds.
We deployed our robots in our experimental platform to demonstrate and verify the result.
The maximum time between any two successive data uploads was measured to be 48 seconds.
The video available at {\footnotesize\url{http://hyness.bu.edu/lomap/dars2012.mov}} demonstrates the execution of this run by the robots.

\section{Conclusion}
\label{sec:conc}
In this paper we presented an automated method for planning optimal paths for a robotic team subject to temporal logic constraints expressed in LTL.
The robots that we consider have bounded non-deterministic traveling times characterized by robot specific deviation values.
We first compute a set of optimal satisfying paths for the members of the team.
Then, leveraging the communication capabilities of the robots, we also compute a set of synchronization sequences for each robot to ensure that the LTL formula is never violated during deployment.
Our experiments show that our method has practical value in scenarios where the traveling times of the robots during deployment deviate from those used in planning.

\begin{acknowledgement}
This work was supported in part by ONR MURI N00014-09-1051, NSF CNS-0834260, and NSERC.
\end{acknowledgement}

\bibliographystyle{spmpsci}
\bibliography{references}
\end{document}